\documentclass{article}

\usepackage{times}
\usepackage{amsmath}
\usepackage{enumitem}
\usepackage{amssymb}
\usepackage{algorithmic}
\usepackage{algorithm}
\usepackage{fullpage}
\usepackage{bbm}
\usepackage{natbib}
\usepackage[colorlinks = true, linkcolor = blue, citecolor = cyan]{hyperref}
\setcitestyle{authoryear,round,citesep={;},aysep={,},yysep={;}}
\renewcommand{\cite}[1]{\citep{#1}}

\usepackage{nicefrac}
\usepackage{microtype}

\usepackage{amsthm}
\usepackage{amsfonts}
\usepackage{comment}
\usepackage{graphicx}
\usepackage{hyperref}
\usepackage{float} 
\usepackage{color}

\include{defs}
\def\shownotes{1}  %set 1 to show author notes
\ifnum\shownotes=1
\newcommand{\authnote}[2]{{$\ll$\textsf{\footnotesize #1 notes: #2}$\gg$}}
\else
\newcommand{\authnote}[2]{}
\fi

\newtheorem{thm}{Theorem}[section]
\newtheorem{lem}{Lemma}[section]

\newtheorem{rmk}{Remark}[section]

\newcommand{\E}{\mathbb{E}}

\newcommand{\He}{\mathit{He}}

\newcommand{\1}{\mathbbm{1}}

\newcommand{\sgn}{\text{sgn}}
\newcommand{\supp}{\text{supp}}
\begin{document}
\title{Representational Power of ReLU Networks and Polynomial Kernels: Beyond Worst-Case Analysis}
\author{Frederic Koehler \thanks{Department of Mathematics, Massachusetts Institute of Technology. Email: fkoehler@mit.edu}, Andrej Risteski \thanks{Department of Mathematics and IDSS, Massachusetts Institute of Technology. Email:risteski@mit.edu}}
\maketitle

%\vspace{0cm}
\begin{abstract}
  
	There has been a large amount of interest, both in the past and
  particularly recently, into the power of different families of universal approximators, 
	e.g. ReLU networks, polynomials, rational functions. 
	However, current research has focused
  almost exclusively on understanding this problem in a
  \emph{worst-case setting}, e.g. bounding the error of the best
  infinity-norm approximation in a box. In this setting a high-degree polynomial
	is required to even approximate a single ReLU. 
	
	However, in real applications with high dimensional data we
  expect it is only important to approximate the desired function well
  on certain relevant parts of its domain. With this motivation, we analyze
  the ability of neural networks and polynomial kernels of bounded
  degree to achieve good statistical performance on a simple, natural
  inference problem with sparse latent structure. We give almost-tight
  bounds on the performance of both neural networks and low degree
  polynomials for this problem. Our bounds for polynomials involve new
  techniques which may be of independent interest and show major
  qualitative differences with what is known in the worst-case
  setting.
\end{abstract}

\section{Introduction} 

The concept of representational power has been always of great
interest in machine learning. In part the reason for this is that
classes of ``universal approximators'' abound -- e.g. polynomials,
radial bases, rational functions, etc. Some of these were known to
mathematicians as early as Bernstein and Lebesgue\footnote{
  Lebesgue made use of the universality of absolute value and hence ReLu -- see the
introduction of \cite{newman1964rational}.} -- yet it is apparent that
not all such classes perform well empirically.

In recent years, the class of choice is neural networks -- which have
inspired a significant amount of theoretical work. Research has focus
on several angles of this question, e.g. comparative power to other
classes of functions \cite{yarotsky2017error,safran2017depth,telgarsky2017neural} % add more citations
% and rational functions \cite{telgarsky2017neural},
the role of depth
and the importance of architecture \cite{telgarsky2016benefits,safran2017depth,eldan2016power},
and many other topics such as their generalization properties and choice
of optimization procedure \cite{hardt2015train,zhang2016understanding,bartlett2017spectrally}.

Our results fall in the first category: namely, comparing the relative
power of polynomial kernels and ReLU networks -- but with a
significant twist. The flavor of existing results that compare
different classes of approximators is approximately the following:
every predictor in a class $\mathcal{C}_1$ can be approximately
represented as a predictor in a different class $\mathcal{C}_2$, with
some blowup in the size/complexity of the predictor (e.g. degree,
number of nodes, depth).

The unsatisfying aspect of such results is the way approximation is
measured: typically, one picks a domain relevant for the approximation
(e.g. an interval or a box), and considers the
$L_{\infty}, L_2, L_1, \dots$ norm of the difference between
the two predictors on this domain. This is an inherently
``worst-case'' measure of approximation: it's quite conceivable that
in cases like multiclass classification, it would suffice to
approximate the predictor well only on some ``relevant domain'',
e.g. far away from the prediction boundary.

The difficulty with the above question is that it's not always easy to formalize what the ``relevant domain'' is, especially without modeling the data distribution. We tackle here arguably the easiest nontrivial incarnation of this question: namely, when there is \emph{sparse} latent structure.% along with a \emph{sparsity} structure.

\section{Overview of results} 

We will be considering a very simple regression task, where the data has a latent sparse structure. More precisely, our setup is the following.
We wish to fit pairs of (observables, labels) $(X, Y)$ generated by a (\emph{latent-variable}) process: 
\begin{itemize}
\item First we sample a \emph{latent vector} $Z \in \mathbb{R}^m$ from $\mathcal{H}$, where $\mathcal{H}$ is a distribution over sparse vectors.
\item To produce $X \in \mathbb{R}^n$, we set $X = AZ + \xi$, where the noise $\xi \sim subG(\sigma^2)$ is a subgaussian random vector with variance proxy $\sigma^2$ (e.g. $N(0, \sigma^2 I)$). 
\item To produce $Y \in \mathbb{R}$, we set $Y = \langle w, Z\rangle$.  
\end{itemize}
 
We hope the reader is reminded of classical setups like sparse linear
regression, compressive sensing and sparse coding: indeed, the
distribution on the data distribution $X$ is standard in all of these
setups. In our setup additionally, we attach a regression task to this
data distribution, wherein the labels are linearly generated by a
predictor $w$. (One could also imagine producing discrete labels by
applying a softmax operation, though the proofs get more difficult in
this case, so we focus on the linear case.)

Our interest however, is slightly different than usual. Typically, one
is interested in the statistical/algorithmic problem of inferring $Z$,
given $X$ as input (the former studying the optimal rates of
``reconstruction'' for $Z$, the latter efficient algorithms for doing
so). Therefore, one does not typically care about the particular
form of the predictor as long as it is efficiently computable.

In contrast, we will be interested in the \emph{representational
  power} of different types of predictors which are commonly used in
machine learning, and their relative \emph{statistical power} in this
model. In other words, we want to understand how close in average-case error we can get to the optimal
predictor for $Y$ given $X$ using standard function classes in machine learning. Informally, what we will show is the following.
\begin{thm}[Informal] For the problem of predicting $Y$ given $X$ in the generative model for data described above, it holds that: \\  
(1) Very simple two-layer ReLU networks achieve close to the statistically optimal rate. \\ 
(2) Polynomial predictors of degree lower than $\log n$ achieve a statistical rate which is substantially worse. 
(In fact, in a certain sense, close to ``trivial''.) Conversely, polynomial predictors of degree $(\log n)^2$ achieve close to the statistically optimal rate.  
\end{thm} 
In particular, if we consider fitting a polynomial to data points of the form
$(x_i,y_i)$, we would need to search through the space of multivariate
polynomials of degree $\Omega(\log n)$ which has super-polynomial
dimension $n^{\Omega(\log(n))}$, and thus even writing down all of the
variables in this optimization problem takes super-polynomial
time. Practical aspects of using polynomial kernels even with
much lower degree than this have been an important concern
and topic of empirical research; see for example \cite{chang2010training} and references
within\footnote{In the {\bf classification} context, the \emph{kernel trick} is a popular way to optimize even over infinite-dimensional
  spaces of functions. In this case, the ability of algorithms to efficiently find
  a good function depends on more complex interactions between the particular
  choice of kernel and the well-separatedness of the data; we leave analyzing these kinds of quantities for future work.}. 
	
	Note that our lower bound holds even though
very simple ReLU networks can easily solve this regression problem.
On the other hand, our upper bound shows that this
analysis is essentially tight: greater than $polylog(n)$
degree is not required to achieve good statistical performance,
which is much different from the situation under worst-case analysis
(see Section~\ref{sec:upper-bound-explanation}). 

For formal statements of the theorems, see Section~\ref{sec:main-results}. 
 
\section{Prior Work}
There has been a large body of work studying the ability of neural
networks to approximate polynomials and various classes of
well-behaved functions, such as recent work
\cite{yarotsky2017error,safran2017depth,telgarsky2017neural,poggio2017}.
As described in the introduction, these results all focus on the
worst-case setting where the goal is to find a network close to some
function in some norm (e.g. infinity-norm or 1-norm). 
There has been
comparatively little work on the problem of approximating ReLU
networks by polynomials: mostly because it is well-known by classical results
of approximation theory
\cite{newman1964rational,constructive-approximation} that polynomials
of degree $\Omega(1/\epsilon)$ are \emph{required} to approximate even a
single absolute value or ReLU function within error $\epsilon$ in
infinity-norm on the interval $[-1,1]$. In contrast
to this classical result, we will show that if we do not seek
to achieve uniform $\epsilon$-error everywhere for the ReLU (in particular not near the non-smooth point at $0$)  we
can build good approximations to ReLU using polynomials of degree
only $O(\log^2(1/\epsilon))$ (see discussion in Section~\ref{sec:upper-bound-explanation}
and Theorem~\ref{thm:relu-approximant}).

Due to the trivial $\Omega(1/\epsilon)$ lower bound for worst-case
approximation of ReLU networks by polynomials,
\cite{telgarsky2017neural} studied the related problem of
approximating a neural network by rational functions. (A classical result of approximation theory
\cite{newman1964rational} shows that rational functions of degree
$O(\log^2(1/\epsilon))$ can get within $\epsilon$-error of the
absolute value function.) %In \cite{telgarsky2017neural} bounds are
%given for approximating neural networks by rational functions and vice
%versa in $\ell_{\infty}$-norm on a box $[-1,1]^n$. The conversion from
%neural networks to rational functions proceed by plugging in a
%rational function approximation to each ReLU and bounding the
%resulting error; 
In particular, \cite{telgarsky2017neural} shows that rational functions of degree
$polylog(1/\epsilon)$ can get within $\epsilon$ distance in
$L_{\infty}$-norm of bounded depth ReLU neural networks.

Somewhat related is also the work of 
\cite{livni-et-al} who considered neural networks with quadratic
activations and related their expressivity to that of sigmoidal
networks in the depth 2 case by building on results of
\cite{shalev2011learning} for approximating sigmoids. The result in 
\cite{shalev2011learning} is also proved using complex-analytic tools, though the details are
different (in particular, they do not use Bernstein's theorem).
%, in contrast
%to the trivial $\Omega(1/\epsilon)$ lower bound for conversion to
%polynomials alluded to above.

There is a vast literature on high dimensional
regression and compressed sensing which we do not attempt to survey,
since the main goal of our paper is \emph{not} to develop new techniques
for sparse regression but rather to analyze the representation power
of kernel methods and neural networks. Some relevant references for sparse recovery
can be found in \cite{vershynin2016high,rigollet-hdstats}.
We emphasize that the upper bound via soft thresholding we show (Theorem~\ref{thm:relu-upperbound-intro}), is implicit in the literature on
high-dimensional statistics; we include the proofs here solely
for completeness.
%We note that this method can be
%thought of as a ``one-step'' version of the LASSO when it is
%computed via iterative thresholding (see
%\cite{beck2009fast}).

\section{Main Results}\label{sec:main-results}
  In this section we will give formal statements of the results, explain
  their significance in detail, and give some insight into the techniques
  used.
  
  % Let us first formally state the results.
 First, let us state the assumptions on the parameters of our generative model:
\begin{itemize}
\item We require $Z$ to be sparse; more precisely we will require that $|\supp(Z)| \le k$ and $\|Z\|_1 \le M$ with high probability.\footnote{The assumed 1-norm bound $M$ plays a minor role in our bounds and is only used when the incoherence $\mu > 0$.}
\item We assume that $A$ is a $\mu$\emph{-incoherent} $n \times m$ matrix,
  which means that %\footnote{Here $\|A\|_{\infty}$ denotes the max-norm; see Notation.}
  $\|A^{\top} A - I\|_{\infty} \le \mu$ for
  some $\mu \ge 0$.
\item We assume (without loss of generality, since changing the magnitude of $w$ just rescales $Y$) that $\|w\|_{\infty} = 1$., 
\end{itemize}

The assumption on $A$ is standard in the literature on sparse recovery (see reference texts \cite{rigollet-hdstats,moitra2014algorithmic}).
In general one needs an assumption like this (or a stronger one, such as the RIP property) in order to guarantee that standard
algorithms such as LASSO actually work for sparse recovery. Note that such an $A$ can be produced e.g. by taking a matrix with i.i.d. entries of the form $\pm 1/\sqrt{n}$ where the
sign is picked randomly and this is possible even when $m >> n$; furthermore the resulting $\mu$ is quite small ($O(1/\sqrt{n})$).

For notational convenience, we will denote $\|A\|_{\infty} = \max_{i,j} |A_{i,j}|$.
\begin{comment}
\begin{itemize}
%\item The \emph{latent distribution} $\mathcal{H}$ over $\mathbb{R}^m$ generates $k$-sparse vectors, i.e. if $Z \sim \mathcal{H}$ then $|\supp(Z)| \le k$ almost surely. %We also require that the variance of $Z_i$ conditioned in $i \in supp(Z)$ is at most $\gamma^2$ for some constant $\gamma^2$. % need some upper bound to deal with interference in the incoherent case.
  %with high probability: more specifically we require that $\Pr_{h \sim H}[|\supp(h)| > k] \le \nu$ for some small $\nu \ge 0$.  
\item $A$ is an $\mu$-incoherent $n \times m$ matrix such that $\|A^{\top} A - Id\|_{\infty} \le \mu$ for some $\mu > 0$. 
  % \item We will assume $w$ is dense in the following sense: for any subset $S$ of size $k$, $\|w|_S\|_1 \le \Delta k \|w\|_2/\sqrt{m}$. Without loss of generality we will assume $\|w\|_2 = \sqrt{m}$ so we just have $\|w|_S\|_{1} \le \Delta k$.
  %\item We assume some bound on the size of $w$ is dense in the following sense: for any subset $S$ of size $k$, $\|w|_S\|_1 \le \Delta k$. Without loss of generality we will assume $\|w\|_2 = \sqrt{m}$ so we just have $\|w|_S\|_{1} \le \Delta k$.  
\end{itemize}
\end{comment}

We proceed to the results: 

\subsection{Regression Using 2-layer ReLU Networks} 

Proceeding to the upper bounds, we prove the following theorem, which shows that 2-layer ReLU networks can achieve an almost optimal statistical rate. Let us 
denote the soft threshold function with threshold $\tau$ as $\rho_{\tau}(x) := \sgn(x)\tau \min(0, |x| - \tau) = \mbox{ReLU}(x - \tau) + \mbox{ReLU}(-x + \tau)$.

Consider the following estimator (for $y$), corresponding to a 2-layer neural network:
\[ \hat{Z}_{NN} := \rho_{\tau}^{\otimes n}(A^{\top} X) \]
\[ \hat{Y}_{NN} := \langle w, \hat{Z}_{NN} \rangle \]
We can prove the following result for this estimator (see Appendix~\ref{sec:relu-upperbound} of the supplement): 

\begin{thm}[2-layer ReLU] \label{thm:relu-upperbound-intro}
Assume $A$ is $\mu$-incoherent. With high probability, the estimator  $\hat{Y}_{NN}$ satisfies 
\[ (\hat{Y}_{NN} - Y)^2 = O((1 + \mu) \sigma^2 k^2 \log(m) + \mu^2 k^2 M^2) \]
\end{thm}
In order to interpret this result, recall that one typically considers 
incoherent matrices where $\mu$ is quite small -- in particular $\mu \ll 1$. Thus the error of the
estimator is essentially $O(\sigma^2 k^2 \log(m))$, i.e. $\sigma^2$ error ``per-coordinate''. It can be shown that this upper bound is nearly information-theoretically optimal (see Remark~\ref{rmk:info-optimal}).

\subsection{Regression Using Polynomials} 

\subsubsection{Lower Bounds for Low-Degree Polynomials} 

We first show that polynomials of degree smaller than
$O(\log n)$ essentially cannot achieve a ``non-trivial'' statistical
rate. This holds even in the simplest possible case for the dictionary $A$:
when it's the identity matrix. 

More precisely, we consider the
situation in which $A$ is an orthogonal matrix (i.e. $\mu = 0, m = n$),
$w \in \{\pm 1\}^n$, the noise distribution is Gaussian $N(0,\sigma^2 I)$,
and the entries of $Z$ are independently 0 with probability
$1 - k/n$ and $N(0,\gamma^2)$ with probability $k/n$. Then we show
\begin{thm}\label{thm:poly-lower-bound}
  Suppose $k < n/2$ and $f$ is a multivariate degree $d$ polynomial. Then
\[ \E[(f(X) - Y)^2] \ge (1/4) \frac{\gamma^2 k}{\left(1 + \sqrt{k/n}(d + 1)^{3d + 2} \left(1 + (\gamma/\sigma)^d\right)\right)^2} \]
\end{thm}
In order to parse the result, observe that the numerator is of order
$\gamma^2 k$ which is the error of the trivial estimator and
the denominator is close to $1$ unless $d$ is sufficiently large with
respect to $n$.
%,so this result says that low-degree polynomials make
%error that differs only by a constant factor from that of the trivial
%estimator.
%We make this statement more precise by looking at the
%behavior of the rhs in the high-dimensional limit of $n \to \infty$.
More precisely, assuming %$k$ is a fixed constant and 
the signal-to-noise ratio $\gamma/\sigma$ does not grow too quickly with respect to $n$, we see
that the denominator is close to 1 unless $d^{d} = \Omega(\sqrt{n})$,
i.e. unless $d$ is of size $\Omega((\log n)/\log \log n)$.

%Note that in order to just writing down the regression problem with a
%multivariate polynomial of degree $d$ requires $\Theta(n^d)$ many
%variables; therefore if $d = \omega((\log n)/\log \log n)$ it takes
%super-polynomial time to write down and solve the kernel regression
%problem of finding $f$ given data, without even considering what
%kind of sample complexity is needed to fit such high-degree polynomials.

%We then use Fourier analysis on orthogonal polynomial
%expansions in order to show that polynomials with large coefficients are
%very noise-sensitive and thus perform poorly in the objective; however,
%if the coefficients are very small then the estimator is not sensitive enough
%to the signal to perform successful reconstruction. Essentially, analyzing
%this tradeoff gives the result.

\subsubsection{Nearly Matching Upper Bounds via Novel Polynomial Approximation to ReLU}\label{sec:upper-bound-explanation}
The lower bound of the previous section leaves open the possibility
that polynomials of degree $O(\mbox{polylog}(n))$ still do not suffice to
perform sparse regression and solve our inference problem; Indeed, it
is a well-known fact (see e.g. \cite{telgarsky2017neural}) that to
approximate a single ReLU to $\epsilon$-closeness in infinity norm in
$[-1,1]$ requires polynomials of degree $\mbox{poly}(1/\epsilon)$; this
follows from standard facts in approximation theory
\cite{constructive-approximation} since ReLU is not a smooth function.

Since estimates for $Y$ typically accumulate error from estimating
each of the $n$ coordinates of $Z$, one would expect we need
$\epsilon = \mbox{poly}(1/n)$ and this suggests that $poly(n)$-degree
polynomials may be needed to build a multivariate polynomial
with similar statistical performance to the 2-layer ReLU network
which computes $\hat{Y}_{NN}$.

Surprisingly, we show this intuition is incorrect! In fact, we show
how to convert the neural network $\hat{Y}_{NN}$ to a $\mbox{polylog}(n)$ degree polynomial with similar statistical
performance by designing a new low-degree polynomial approximation to
ReLU. Formally this is summarized by the following theorem, where $\hat{Y}_{d,M}$
is the corresponding version of $\hat{Y}_{NN}$ formed by replacing
each ReLU by our polynomial approximation.
\begin{thm}\label{thm:poly-upper-bound-intro}
  Suppose 
  $\tau = \Theta(\sigma \sqrt{(1 + \mu) \log m} + \mu M)$ and
  $d \ge d_0 = \Omega((2 + \frac{M}{\tau}) \log^2(M m/\tau^2))$.
With high probability, the estimator  $\hat{Y}_{d,M}$ satisfies 
\[ (\hat{Y}_{d,M} - Y)^2 = O(k^2 ((1 + \mu) \sigma^2 \log(m) + \mu^2 M^2)) \]
\end{thm}

%the intuition is that if we want to build a neural
%network which is reasonably tolerant to random variation in inputs,
%there need to be regions in our non-linearity which are not sensitive
%so that the noise does not propagate through to the output of our
%network. We verify this intuition in the particular case of our sparse
%regression setup, where we show that plugging in our new polynomial
%approximant in place of ReLU does indeed give polynomials which are able
%to recover $Z$ and $Y$ with similar performance to the neural network.

\section{Techniques}

We briefly survey each of the results above. 

Proceeding to the upper bound using a 2-layer ReLU network, results similar to Theorem~\ref{thm:relu-upperbound-intro} are
standard in the literature on sparse linear regression,
though we include the proof for completeness in
Appendix~\ref{sec:relu-upperbound} of the supplement. The intuition is
quite simple: the estimator $\hat{Z}_{NN}$ can make use of the
non-linearity in the soft threshold to zero out the coordinates in the
estimate $A^{\top} X$ which are small and thus ``reliably'' not in the
support of the true $z$. This allows the estimator to only make
mistakes on the non-zero coordinates. %the worst situation is when
%one of the non-zero coordinates is of the same size as the maximum
%of the noise, which of order $\sigma \sqrt{\log m}$.
%Thus we achieve an error of only $O(\sigma \sqrt{\log(m)})$ per
%coordinate; computing $Y$ adds up the error from each of these coordinates
%and gives the theorem. 
%optimal even in the case $A$ is the identity matrix, in which case the optimal error
%for estimating $Y$ is directly related to the optimal error for recovering $Z$,
%which is a sparse regression problem for which lower bounds are well-known (see e.g. reference text
%\cite{rigollet-hdstats}); this is further explained in Remark~\ref{rmk:info-optimal}

For lower bound on low-degree polynomial kernels, we first make a probabilistic argument
which lets us reduce to studying the coordinate-wise error
in reconstructing $Z$, and subsequently use Fourier analysis on orthogonal polynomials
to get a bias-variance tradeoff we can lower bound. \footnote{On a
technical note we observe that this statement is given with respect to
expectation but a similar one can be made with high probability, see
Remark~\ref{rmk:expectation-vs-whp}.}

We prove the upper bound for polylogarithmic degree kernels by constructing a
new polynomial approximation to ReLU (Theorem~\ref{thm:relu-approximant}).  The key insight here is that
approximating the ReLU in infinity-norm is difficult in large part
because it is hard to approximate ReLU at 0, its point of non-smoothness; however, in our case the precise behavior of ReLU very close to 0
is not important for getting a good regression rate. Instead, the
polynomial approximation to ReLU we design uses only $O(\log^2 n)$
degree polynomials and sacrifices optimizing accuracy in approximation
near the point of non-smoothness in favor of optimizing closeness to 0
in the negative region.  The reason that closeness to 0 is so
important is that it captures the ability of ReLU to \emph{denoise},
since the 0 region of ReLU is insensitive to small changes in the
input; essentially all commonly used neural network
activations have such a region of insensitivity (e.g. for a sigmoid,
the region far away from 0).

Our polynomial approximation to ReLU
is built using powerful complex-analytic tools from approximation
theory and may be of independent interest;
we are
not aware of a way to get this result using only generic techniques
such as FT-Mollification \cite{ft-mollification}.
%; see Theorem~\ref{thm:relu-approximant}.  Part of our proof uses a special
%mollification to ReLU with good complex-analytic properties and we are
%not aware of a way to get this result using only generic techniques
%such as FT-Mollification \cite{ft-mollification}.
%\subsection{Notation}

%Our techniques are interesting in their own right: the 

%In this paper, we attach the problem from a different point of view: namely, we wish to compare the representational power of different classes when data has \emph{a natural distribution}attached to it. In fact, here we attack a very natural distribution inspired by sparse linear regression and compressed sensing.  

%\Anote{More intro.} 

%%% Local Variables:
%%% mode: latex
%%% TeX-master: "main"
%%% End:

%\section{Upper bound via 2-Layer ReLu Network}
%This section is deferred to Appendix A of the supplement.
% \input{setup}
%\input{relu-upper-bound}

\section{Part 1: Lower Bounds for Polynomial Kernels}

In this section, we fill flesh out the lower bound results somewhat more. 
 
The lower bound instance is extremely simple: the noise distribution is $N(0,\sigma^2 Id)$
and the distribution for $Z$ is s.t. every coordinate is first chosen to be non-zero with probability $k/n$, and if it is non-zero, it's set as an independent sample from $N(0,\gamma^2)$.  

This construction makes $Z$ approximately $k$-sparse with
high probability while making its coordinates independent. We choose $A$
as an arbitrary orthogonal matrix, so $m = n$. We choose $w$
to be an arbitrary $\pm 1$ sign vector, so $w_i^2 = 1$ for every $i$.

%Note that there is always a trivial estimator which return 0 and
%makes error $\gamma^2 k$, and it can be shown estimator is essentially optimal
%unless the signal is larger than the maxima of the noise, i.e. we need that
%$\gamma >> \sigma \sqrt{\log n}$. When $\gamma >> \sigma \sqrt{\log n}$,
%e.g. $\gamma = \sigma (\log n)$, then thresholding achieves an error
%of $\sigma^2 k \log n << \gamma^2 k$ which is nontrivial. The thresholding achieves the information theoretic rate in this regime (see Remark~\ref{rmk:info-optimal}).

We will show first that linear predictors, and
then fixed low degree polynomials cannot achieve the
information-theoretic rate\footnote{See Remark~\ref{rmk:info-optimal}} of $O(\sigma^2 k \log n)$ -- in fact, 
we will show that they achieve a ``trivial'' rate. 
Furthermore, we will
show that even if the degree of our polynomials is growing with $n$,
if $d = o(\log n/\log \log n)$ the state of affairs is similar. 

 %then it is still not possible to
%achieve the information-theoretic rate for the square loss,
%$O(\sigma^2 k \log n)$ -- and in fact, the  
\subsection{Warmup: Linear Predictors} 

%We prove that low-degree kernels cannot achieve a good statistical rate. 
As a warmup, and to illustrate the main ideas of the proof techniques,
we first consider the case of linear predictors. (i.e. kernels of degree 1.) 

The main idea is to use a bias-variance trade-off: namely, we show that the linear predictor we use, say $f(x) = \langle \tilde{w},x \rangle$ 
either has to have too high of a variance (when $\|\tilde{w}\|$ is large), or otherwise has too high of a bias. (Recall, the bias captures how well 
the predictor captures the expectation.)  

We prove:
\begin{thm}\label{thm:linear-lower} For any $\tilde{w} \in \mathbb{R}^n$,
  \[   \E [(\langle \tilde{w}, X \rangle - Y)^2] \ge \gamma^2 k \frac{\sigma^2}{\gamma^2 (k/n) + \sigma^2} \]
\end{thm}

Before giving the proof, let us see how the theorem should be
interpreted. 

The trivial estimator which always returns 0 makes error of
order $\gamma^2 K$ and a good estimator (such as thresholding) should
instead make error of order $\sigma^2 K \log n$ when
$\gamma >> \sigma\sqrt{\log n}$. The next theorem shows that as long
as the signal to noise ratio is not \emph{too high}, more specifically
as long as $\gamma^2(k/n) = o(\sigma^2)$, any linear estimator must make
square loss of $\Omega(\gamma^2 k)$, i.e. not significantly better than the trivial 0 estimate. 

Note that the most interesting (and difficult) regime is when the signal is not too much larger than
the noise, e.g. $\gamma^2 = \sigma^2 \mbox{polylog}(n)$ in which case it is definitely true
that $\gamma^2 (k/n) << \sigma^2$.

\begin{proof}%[Proof of Theorem~\ref{thm:linearlower}]
  Note that 
  \begin{align*} \langle \tilde{w}, x \rangle - y
    = \langle \tilde{w}, A z + \xi \rangle - \langle w, z \rangle 
    = \langle A^{\top} \tilde{w} - w, z \rangle + \langle \tilde{w}, \xi \rangle
\end{align*}
which gives the following \emph{bias-variance decomposition} for the square loss:
\begin{align*}
  \E [(\langle \tilde{w}, x \rangle - y)^2 ]
  &= \E [(\langle A^{\top} \tilde{w} - w, z \rangle + \langle \tilde{w}, \xi \rangle)^2 ]   \\
  &= \E [\langle A^{\top} \tilde{w} - w, z \rangle^2 + \langle \tilde{w}, \xi \rangle^2 ] \\
  &= \frac{k}{n} \gamma^2 \|A^{\top} \tilde{w} - w\|_2^2 + \sigma^2 \|\tilde{w}\|_2^2 \\
  &= \frac{k}{n} \gamma^2 \|\tilde{w} - A w\|_2^2 + \sigma^2 \|\tilde{w}\|_2^2  
\end{align*}
where in the second-to-last step we used that the covariance matrix of $Z$ is   $\gamma^2 (k/n) I$, and in the last step we used that $A$ is orthogonal. Now observe that if we
fix $R = \|\tilde{w}\|_2$, then by the Pythagorean theorem the
minimizer of the square loss is given by the projection of $A w$
onto the $R$-dilated unit sphere, so $\tilde{w} = \sqrt{R^2/m}(A w)$ since
$\|A w\|_2 = \|w\|_2 = \sqrt{m}$.  In this case the square loss is then of the form
\[ \frac{k}{n} \gamma^2  \|\sqrt{R^2/m}(A w) - A w\|_2^2 + \sigma^2 \|\tilde{w}\|_2^2 = \frac{k}{n}\gamma^2 (R - \sqrt{m})^2 + \sigma^2 R^2  \]
and the risk is minimized when
\[ 0 = 2\frac{k}{n}\gamma^2 (R - \sqrt{m}) + 2\sigma^2 R \]
i.e. when
\[ R = \frac{\gamma^2 (k/n)}{\gamma^2 (k/n) + \sigma^2} \sqrt{m} \]
so the minimum square loss is
\[ (\sqrt{m} - R) \sigma^2 R + \sigma^2 R^2 = \sigma^2 \frac{\gamma^2 k}{\gamma^2 (k/n) + \sigma^2} \]
since $m = n$.

\end{proof}

%%% Local Variables:
%%% mode: latex
%%% TeX-master: "main"
%%% End:

%%% Local Variables:
%%% mode: latex
%%% TeX-master: "main"
%%% End:

\subsection{Main Technique for General Case: Structure of the Optimal Estimator}
%\Fnote{Should we add a remark that this kind of argument
%  also shows that the information-theoretic rate should be $\sigma^2 k$
%  via standard results on regression? It is also possible to compute
%this out I think.}
Before proceeding to the proof of the lower bound for general low-degree polynomials, we
observe that the optimal estimator for $Y = \langle w, Z \rangle$
given $X$ has a particularly simple structure. More precisely the
optimal estimator in the squared loss is the conditional expectation,
$\E[\langle w, Z \rangle | X] = \sum_i w_i \E[Z_i | X]$ so the optimal
estimator for $Y$ simply reconstructs $Z$ as well as possible
coordinate-wise and then takes an inner product with $w$. 

In our setup
the coordinates of $Z$ are independent, which allows us to show that
the optimal polynomial of degree $d$ to estimate $Y$ has no ``mixed
monomials'' when we choose the appropriate basis. This is the content of the next lemma.
\begin{lem}\label{lem:no-mixed-monomials}
  Suppose $X = A Z + \xi$ where $A$ is an orthogonal $m \times m$ matrix,
  $Z$ has independent entries and $\xi \sim N(0,\sigma^2 Id)$. 
  Then there exists a unique minimizer $f^*_d$ over all degree $d$ polynomials $f_d$ of the square-loss,
  \[ \E[(f_d(A^{\top} X) - \langle w, Z \rangle)^2]  \]
  and furthermore $f^*_d$ has no mixed monomials. In other words,
  we can write $f^*_d(A^{\top} X) = \sum_i f^*_{d,i}( (A^{\top} X)_i)$ where each
  of the $f^*_{d,i}$ are univariate degree $d$ polynomials.
\end{lem}
\begin{proof}
Let $X' = A^{\top} X$, so by orthogonality $X' = Z + \xi'$ where $\xi' \sim N(0, \sigma^2 Id)$. Observe that if we look at the optimum over all functions $f$, we see that
\begin{align*}
  \min_f\E[(f(X') - \langle w, Z \rangle)^2]
  &= \E[(\E[\langle w, Z \rangle | X'] - \langle w, Z \rangle)^2] \\
  &= \E[(\sum_i w_i \E[Z_i | X'] - \langle w, Z \rangle)^2] \\
  &= \E[(\sum_i w_i \E[Z_i | X'_i] - \langle w, Z \rangle)^2].
\end{align*}
where where in the first step we used that the conditional expectation
minimizes the squared loss,
in the second step we used linearity of conditional expectation,
and in the last step we used that $Z_i$ is independent of $X'_{\ne i}$.

By the Pythagorean theorem,
the optimal degree $d$ polynomial $f^*_d$ is just the projection of
$\sum_i w_i \E[Z_i | X'_i]$ onto the space of degree $d$ polynomials.
On the other hand observe that 
\[ \E[(\sum_i w_i \E[Z_i | X'_i] - \langle w, Z \rangle)^2] = \sum_i w_i^2 \E[(\E[Z_i | X'_i] - Z_i)^2] \]
so the optimal projection $f^*_d$ is just $\sum_i w_i f^*_{i,d}(X'_i)$ where $f^*_{i,d}$ is just the projection of each of the $\E[Z_i | X'_i]$. Therefore $f^*_d$ has no mixed monomials.
\end{proof}
\begin{rmk}\label{rmk:info-optimal}
  The previous calculation shows that the problems of minimizing the
  squared loss for predicting $Y$ is equivalent to that of minimizing
  the squared loss for the sparse regression problem of recovering
  $Z$. It is a well-known fact that the
  information theoretic rate for sparse regression (with our
  normalization convention) is $\Theta(\sigma^2 k)$ (see for example
  \cite{rigollet-hdstats}), and so the information-theoretic rate for
  predicting $Y$ is the same, and is matched by
  Theorem~\ref{thm:predicting-Y}. In our particular model it is also
  possible to compute this directly, since one can find an
  explicit formula for $\E[Z_i | X'_i]$ using Bayes rule.
\end{rmk}

\subsection{Lower Bounds for Polynomial Kernels}

The lower bound for polynomials combines the observation
of Lemma~\ref{lem:no-mixed-monomials} with a more general
analysis of bias-variance tradeoff using Fourier
analysis on orthogonal polynomials. Concretely, since the noise we chose for the
lower bound instance is Gaussian, the most convenient basis will be the Hermite polynomals. 

%though it will be more convenient to work in a different basis -- 
%more concretely, given that the noise is Gaussian, it will be most convenient to work in with the basis of Hermite polynomials. We fix $w$ to be an arbitrary vector with $\pm 1$ entries.

Recall that the \emph{probabilist's Hermite polynomial} $He_n(x)$ can be defined by the recurrence relation
\begin{equation}\label{eqn:He-recurrence}
  \He_{n + 1}(x) = x \He_n(x) - n \He_{n - 1}(x).
\end{equation}
where $\He_0(x) = 1, \He_1(x) = x$.
In terms of this, the \emph{normalized Hermite polynomial} $H_n(x)$ is
\[ H_n(x) = \frac{1}{\sqrt{n!}} \He_n(x). \]
Let $H_{\mathbf{n}}(x)$ for a vector of indices $\mathbf{n} \in \mathbb{N}^m_0$ denote the multivariate polynomial 
$\Pi_{i=1}^m H_{\mathbf{n}_i} (x_i)$. It's easy to see the polynomials $H_{\mathbf{n}}(x)$ form an orthogonal basis with respect to the standard m-variate Gaussian distribution. 
As a consequence, we get 
$$\E_{X \sim \mathcal{N}(0, \sigma^2 I}) H_{\mathbf{n}}(X/\sigma) H_{\mathbf{n'}}(X/\sigma) =   
\begin{cases}
  0, \mbox{ if } \mathbf{n} \neq \mathbf{n'} \\
  1, \mbox{ otherwise}
% doubt this version (without dividing by sigma on lhs) is true...
%(\frac{1}{\sigma^2})^{|\mbox{supp}(\mathbf{n})|}, \mbox{ otherwise}  
\end{cases} 
$$
which gives us Plancherel's theorem:
\begin{thm}[Plancherel in Hermite basis]
  Let $f(x) = \sum_{\mathbf{n}} \widehat{f}(\mathbf{n}) H_{\mathbf{n}}(x/\sigma)$, then
  \[ \E_{X \sim \mathcal{N}(0, \sigma^2 I})[|f(X)|^2] = \sum_{\mathbf{n}} |\widehat{f}(\mathbf{n})|^2 \]
\end{thm}
We can use Plancherel's theorem to get lower
bounds on the noise sensitivity of degree $d$
polynomials. This will be an analogue of the variance. 

\begin{lem}\label{lem:poly-variance} [Variance analogue in Hermite basis]
  Let $f(x) = \sum_{\mathbf{n}} \widehat{f}(\mathbf{n}) H_{\mathbf{n}}(x/\sigma)$
  and let $f_{\ne 0} := f - \widehat{f}(0)$.
  Then
\[ \E [(f(A^{\top} X) - Y)^2 ] \ge (1 - k/n) \|\widehat{f}_{\ne 0}\|_2^2  \]
\end{lem}
\begin{proof}
  First suppose $Z$, and thus $y$, is fixed.  Let $S$ denote the
  support of $Z$. Recall that $A^{\top} x = Z + \xi'$ where
  $\xi' \sim \mathcal{N}(0,\sigma^2 I_{n \times n})$. Define
  $f_Z(\xi) := f(Z + \xi) - y$, then by Plancherel
  \[ \E_{\xi}[(f(A^{\top} x) - y)^2] = \E_{\xi'}[f_Z(\xi')^2] = \sum_{\mathbf{n}} |\widehat{f_Z}(\mathbf{n})|^2 \]
  Furthermore
  \[ \sum_{\mathbf{n}} |\widehat{f_Z}(\mathbf{n})|^2 \ge
    \sum_{\mathbf{n} : supp(\mathbf{n}) \not\subset S}
    |\widehat{f}(\mathbf{n})|^2 \] because
  $(\xi' + Z)|_{S^c} = \xi'|_{S^c}$ so by expanding out $f_Z$ in terms
  of the fourier expansion of $f$, we see
  $\widehat{f_Z}(\mathbf{n}) = \widehat{f}(\mathbf{n})$ for
  $\mathbf{n}$ such that $supp(\mathbf{n}) \not\subset S$.
  Finally the probability $\mathbf{n} \subset S$ for $\mathbf{n} \ne 0$
  is upper bounded by the probability a single element of its support
  is in $S$, which is $k/n$.
\end{proof}

Next we give a lower bound for the bias, showing that if $\|\widehat{f}_{\ne 0}\|_2^2$ is small for a low-degree polynomial, it cannot accurately predict $y$. Here we will assume $f$ is of the form given by Lemma~\ref{lem:no-mixed-monomials}.
\begin{lem}[Low variance implies high bias]\label{lem:poly-bias}
  Suppose $f$ is a multivariate polynomial of degree $d$ with no
  mixed monomials, i.e.  $f(x) = \sum_i f_{i}(x_i)$ where $f_i$ is a univariate polynomial of degree $d$.
  Expand $f$ in terms of Hermite polynomials as $f(x) = \sum_{\mathbf{n}} \widehat{f}(\mathbf{n}) H_{\mathbf{n}}(x/\sigma)$. Then
  % \[  \E [(f(A^{\top} X) - Y)^2 ] \ge 1 - C^d\|\widehat{f}_{\ne 0}\|_2^2/\sigma^2 \]
\[  \E [(f(A^{\top} X) - Y)^2 ] \ge (k/n) \sum_{i = 1}^n w_i^2 \max(0, \gamma - \sqrt{\sum_{i = 1}^n |\hat{f}(k e_i)|^2} (d + 1)^{3d + 2} (1 + (\gamma/\sigma)^d))^2 \]
\end{lem}

Before proving the lemma, let us see how it proves the main theorem: 
\begin{proof}[Proof of Theorem~\ref{thm:poly-lower-bound}]
  By Lemma~\ref{lem:no-mixed-monomials}, Lemma~\ref{lem:poly-variance}, and Lemma~\ref{lem:poly-bias} we have that for the $f$ which minimizes the square
  loss among degree $d$ polynomials, we have a variance-type lower bound
  \[ \E[(f(A^{\top} X) - Y)^2] \ge (1 - k/n) \sum_{i = 1}^n \sum_{k = 1}^d |\hat{f}(k e_i)|^2 \]
  and (using that $w_i^2 = 1$ to simplify) a bias-type lower bound
  \[  \E [(f(A^{\top} X) - Y)^2 ] \ge (k/n) \sum_{i = 1}^n \max(0, \gamma - \sqrt{\sum_{i = 1}^n |\hat{f}(k e_i)|^2} (d + 1)^{3d + 2} (1 + (\gamma/\sigma)^d))^2. \]
  Let $\|\widehat{f_i}\|_2 := \sqrt{\sum_{i = 1}^n |\hat{f}(k e_i)|^2}$. Then
  averaging these lower bounds and simplifying using $k < n/2$ gives
  \begin{align*}
    \E [(f(A^{\top} X) - Y)^2 ]
    &\ge (1/4) \sum_{i = 1}^n \max(\|\widehat{f_i}\|_2, \sqrt{k/n} (\gamma - \|\widehat{f_i}\|_2 (d + 1)^{3d + 2} (1 + (\gamma/\sigma)^d)))^2 \\
    &\ge (1/4) \sum_{i = 1}^n \frac{\gamma^2 (k/n)}{(1 + \sqrt{k/n}(d + 1)^{3d + 2} (1 + (\gamma/\sigma)^d))^2} \\
    &\ge (1/4) \frac{\gamma^2 k}{(1 + \sqrt{k/n}(d + 1)^{3d + 2} (1 + (\gamma/\sigma)^d))^2}
  \end{align*}
\end{proof}

Returning to the proof of Lemma~\ref{lem:poly-bias}, we have:
\begin{proof}[Proof of Lemma~\ref{lem:poly-bias}]
  Since $f$ has no mixed monomials, we get for the Hermite expansion
  that $\widehat{f}(\mathbf{n}) = 0$ unless $|supp(\mathbf{n})| \le 1$.
  Let $X' := A^{\top} X = Z + \xi'$ where $\xi'\sim N(0,\sigma^2 I)$.
  Next observe by independence that
  \[ \E[(f(X') - Y)^2] = \sum_i w_i^2 \E[(f_i(X'_i) - Z_i)^2] \ge (k/n) \sum_i w_i^2\E[(f_i(X'_i) - Z_i)^2 | Z_i \ne 0]  \]
  where the last inequality follows since there is a $k/n$ chance that
  $Z_i \sim N(0, \sigma^2 I)$, equivalently that $Z_i \ne 0$. By the
  conditional Jensen's inequality we have
  \[ (k/n) \sum_i w_i^2\E[(f_i(X'_i) - Z_i)^2 | Z_i \ne 0] \ge (k/n)
    \sum_i w_i^2 \E[(\E[f_i(X'_i) | Z_i] - Z_i)^2 | Z_i \ne 0]. \]
  Observe that $f_i(X'_i) = \sum_{k = 0}^d \widehat{f}(k e_i) H_{k}(Z_i/\sigma +
  \xi/\sigma)$ and let $g_i(Z_i) := \E[f_i(X'_i) | Z_i] - Z_i$,
  so then $g_i$ is a polynomial of degree $d$ in $Z_i$. Write the
  Hermite polynomial expansion of $g_i$ in terms of $H_k(Z_i/\gamma)$ as
  \[ g_i(x) = \sum_{k = 0}^d \widehat{g_i}(k) H_k(Z_i/\gamma), \]
  then by Plancherel's formula
\[ (k/n) \sum_i w_i^2 \E[(\E[g(Z_i) - Z_i)^2 | Z_i \ne 0]  = (k/n) \sum_i w_i^2 \sum_{k = 0}^d |\widehat{g_i}(k)|^2 \ge (k/n) \sum_i w_i^2 |\widehat{g_i}(1)|^2 \]
and it remains to lower bound $|\widehat{g_i}(1)|$. By orthogonality and direct computation,
\[ \widehat{g_i}(1) = \E_{Z_i \sim N(0,\gamma)}[(\E[f_i(X'_i) | Z_i] - Z_i) H_1(Z_i/\gamma)] = -\gamma + \E_{Z_i \sim N(0,\gamma)}[\E[f_i(X'_i) | Z_i](Z_i/\gamma)]. \]
Now we upper bound the last term
\begin{align*}
  \E_{Z_i \sim N(0,\gamma)}[\E[f_i(X'_i) | Z_i](Z_i/\gamma)]
  &= \widehat{f}(0) \E[Z_i/\gamma] + \sum_{k = 1}^d \widehat{f}(k e_i) \E_{Z_i \sim N(0,\gamma)}[ \E[H_{k}(Z_i/\sigma +
    \xi'/\sigma) | Z_i](Z_i/\gamma)] \\
  &= \sum_{k = 1}^d \widehat{f}(k e_i) \E_{Z_i \sim N(0,\gamma)}[H_{k}(Z_i/\sigma +
    \xi'/\sigma)(Z_i/\gamma)] \\
  &\le \left(\sum_{k = 1}^d |\widehat{f}(k e_i)|^2\right)^{1/2} \left(\sum_{k = 1}^d \E_{Z_i \sim N(0,\gamma)}[H_{k}(Z_i/\sigma + \xi'/\sigma)(Z_i/\gamma)]^2\right)^{1/2}
\end{align*}
where the second equality is by the law of total expectation and the last inequality is Cauchy-Schwarz. Using the recurrence relation \eqref{eqn:He-recurrence}, we can bound the sum of the absolute value of the coefficients of $H_k(x)$ by $k^k/\sqrt{k!} \le k^k$. We can also bound the moments of the absolute value of a Gaussian by $\E_{\xi \sim N(0,1)}[|\xi|^k] \le k^k$. Therefore by Holder's inequality
\begin{comment}
\begin{align*}
  \E_{Z_i \sim N(0,\gamma)}[H_{k}(Z_i/\sigma + \xi'/\sigma)(Z_i/\gamma)]
  &\le k^k \sup_{\ell = 1}^k|\E_{Z_i \sim N(0,\gamma)}[(Z_i/\sigma + \xi'/\sigma)^{\ell} (Z_i/\gamma)]| \\
  &= k^k (\sigma/\gamma)\sup_{\ell = 1}^k |\E_{Z_i \sim N(0,\gamma)}[\sum_{j = 1}^{\ell} {\ell \choose j} (Z_i/\sigma)^{j + 1} (\xi'/\sigma)^{\ell - j}]| \\
  &\le k^k (\sigma/\gamma)\sup_{\ell = 1}^k \sum_{j = 1}^{\ell} {\ell \choose j} (j + 1)^{j + 1}(\gamma/\sigma)^{j + 1} (\ell - j)^{\ell - j}
\end{align*}

Q: why do I only have ratios of $\gamma$ and $\sigma$? Obviously the
total error increases if we increase both $\gamma$ and $\sigma$. NVM
this is fine. The point is the error is bounded below by an absolute
constant, and is not shrinking with the sample size? It looks like this is what basically happened in the linear case anyway. (the $\sigma^2$ terms cancelled).
\end{comment}
\begin{align*}
  \E_{Z_i \sim N(0,\gamma)}[H_{k}(Z_i/\sigma + \xi'/\sigma)(Z_i/\gamma)]
  &\le k^k \sup_{\ell = 1}^k|\E_{Z_i \sim N(0,\gamma)}[(Z_i/\sigma + \xi'/\sigma)^{\ell} (Z_i/\gamma)]| \\
  &\le k^k \sup_{\ell = 1}^k\E_{Z_i \sim N(0,\gamma)}[(|Z_i|/\sigma + |\xi'|/\sigma)^{\ell} (|Z_i|/\gamma)] \\
  &\le 2^k k^k \sup_{\ell = 1}^k(\E_{Z_i \sim N(0,\gamma)}[|Z_i|^{\ell + 1}/\sigma^{\ell}\gamma] + \E_{Z_i \sim N(0,\gamma)}[|\xi'|^{\ell} |Z_i|/\sigma^{\ell}\gamma)]) \\
  &\le 2^k k^k [\max(1,(\gamma/\sigma)^{k}) (k + 1)^{k + 1} + k^k] \\
  &\le (k + 1)^{3k + 1} (1 + (\gamma/\sigma)^{k}).
\end{align*}
Therefore by reverse triangle inequality
\[ |\hat{g_i}(1)|^2 \ge \max(0, \gamma - \sqrt{\sum_{i = 1}^n |\hat{f}(k e_i)|^2} (d + 1)^{3d + 2} (1 + (\gamma/\sigma)^d))^2. \]

\end{proof}
%Optimizing the lower bounds gives the result:
%\begin{thm}\label{thm:poly-lower-bound}
%  Suppose $k < n/2$ and $f$ is a multivariate degree $d$ polynomial. Then
%\[ \E[(f(A^{\top} X) - Y)^2] \ge (1/4) \frac{\gamma^2 k}{(1 + \sqrt{k/n}(d + 1)^{3d + 2} (1 + (\gamma/\sigma)^d))^2} \]
%\end{thm}
%\begin{proof}
%  By Lemma~\ref{lem:no-mixed-monomials}, Lemma~\ref{lem:poly-variance}, and Lemma~\ref{lem:poly-bias} we have that for the $f$ which minimizes the square
%  loss among degree $d$ polynomials, we have a variance-type lower bound
%  \[ \E[(f(A^{\top} X) - Y)^2] \ge (1 - k/n) \sum_{i = 1}^n \sum_{k = 1}^d |\hat{f}(k e_i)|^2 \]
%  and (using that $w_i^2 = 1$ to simplify) a bias-type lower bound
%  \[  \E [(f(A^{\top} X) - Y)^2 ] \ge (k/n) \sum_{i = 1}^n \max(0, \gamma - \sqrt{\sum_{i = 1}^n |\hat{f}(k e_i)|^2} (d + 1)^{3d + 2} (1 + (\gamma/\sigma)^d))^2. \]
%  Let $\|\widehat{f_i}\|_2 := \sqrt{\sum_{i = 1}^n |\hat{f}(k e_i)|^2}$. Then
%  averaging these lower bounds and simplifying using $k < n/2$ gives
%  \begin{align*}
%    \E [(f(A^{\top} X) - Y)^2 ]
%    &\ge (1/4) \sum_{i = 1}^n \max(\|\widehat{f_i}\|_2, \sqrt{k/n} (\gamma - \|\widehat{f_i}\|_2 (d + 1)^{3d + 2} (1 + (\gamma/\sigma)^d)))^2 \\
%    &\ge (1/4) \sum_{i = 1}^n \frac{\gamma^2 (k/n)}{(1 + \sqrt{k/n}(d + 1)^{3d + 2} (1 + (\gamma/\sigma)^d))^2} \\
%    &\ge (1/4) \frac{\gamma^2 k}{(1 + \sqrt{k/n}(d + 1)^{3d + 2} (1 + (\gamma/\sigma)^d))^2}
%  \end{align*}
%\end{proof}

\begin{rmk}\label{rmk:expectation-vs-whp} Remarks on results: \end{rmk}

%{\bf Some remarks:} 
We make a few remarks regarding the results in this section. 
Recall that $\gamma^2 k$ is the square loss of the trivial
zero-estimator. Suppose as before that
$\gamma = \Theta(\sigma^2 polylog(n))$, then we see that if
$d = o(\log n/\log \log n)$ then the denominator of the lower bound
tends to $1$, hence any such polynomial estimator has a rate no better
than that of the trivial zero-estimate.

%\begin{rmk}\label{rmk:expectation-vs-whp}
  It is possible to derive a similar statement to
  Theorem~\ref{thm:poly-lower-bound} that holds with high probability
  instead of in expectation for polynomials of degree
  $o(\log n/\log \log n)$. All that is needed is to bound the
  contribution to the expectation from very rare tail events when the
  realization of the noise $\xi$ is atypically large. Since the
  polynomials we consider are very low degree $o(\log n/\log \log n)$,
  they can only grow at a rate of $x^d = x^{o(\log(n)/\log\log n)}$;
  thus standard growth rate estimates (e.g. the Remez inequality)
  combined with the Gaussian tails of the noise can be used to show
  that a polynomial which behaves reasonably in the high-probability
  region (e.g. which has small w.h.p. error) cannot contribute a large
  amount to the expectation in the tail region.
%\end{rmk}

\begin{comment}
THE POINT: this is like the error of the estimator that always plays zero, unless we start to take $d$ larger. Q: DOES IT WORK? The optimal estimator is making error like $\sigma^2 k$, this one is making error like $\gamma^2 k$, like the trivial estimator. Only worry is there is a risk of cancellation with the denominator? I think it works because the $\sqrt{k/n}$ damps the troublesome terms in denominator sufficiently. And we can take $\gamma/\sigma$ to be say $polylog(n)$ without any issue.
\end{comment}

\begin{comment}
We also have the following lower bound which controls the constant term:
\begin{lem}
  Let $f(x) = \sum_{\mathbf{n}} \widehat{f}(\mathbf{n}) H_{\mathbf{n}}(x/\sigma)$  
  then
\[ \E [(f(A^{\top} x) - y)^2 ] \ge \widehat{f}(0)^2 \]  
\end{lem}
\begin{proof}
  Observe that $\E[y] = 0$, therefore 
  \[ \E [(f(A^{\top} x) - y)^2 ] \ge (\E[f(A^{\top} x) - y])^2 =   \]
TODO
\end{proof}
Would this be easier if we use orthogonal polynomials for $x + h$ instead,
as andrej suggested? the fact about exponential bounds on coefficients
should be true in general. But then handling the interaction with $y = \langle  w, h \rangle$ itself
is annoying.
\end{comment}
%%% Local Variables:
%%% mode: latex
%%% TeX-master: "main"
%%% End:

%The full proof is 
%deferred to Appendix~\ref{sec:lower-bound-polynomials} of the
%supplement.
%%% Local Variables:
%%% mode: latex
%%% TeX-master: "main"
%%% End:
 % + short description of polynomial lower bound
%\input{bias-variance}

%%% Local Variables:
%%% mode: latex
%%% TeX-master: "main"
%%% End:
\section{Part 2: A Nearly Optimal Polynomial Construction}
We previously showed that for polynomials to match the statistical
performance of a 2-Layer ReLu network, the degree needs to be
$\Omega(\log n)$. In this section, we show that this is almost tight
by constructing polynomials of degree $O(\log^2 m)$. 

Our strategy is to plug in a good approximation to ReLU in the 2-layer ReLU network
construction. One might hope that simply taking the ``best polynomial
approximation'' of degree $d$ in the typical approximation theory
sense to ReLU in the interval $[-1,1]$ would suffice, but in fact this
is extremely inefficient; because ReLU is not smooth, standard results
in approximation theory (see Chapters 7,8 of \cite{constructive-approximation}) show
that a degree $d$ polynomial cannot get closer than $O(1/poly(d))$ in infinity-norm.
(And as noted before, we don't need to approximate the ReLU particularly well near the kink.)

%A polynomial rate like this would mean we need very high degree polynomials (e.g. $O(\sqrt{m})$)
%to guarantee that our construction had reasonable performance for the regression
%problem.

Instead we will carefully design an approximation to ReLu: in particular,
the polynomial we take will be extremely close to 0 in the threshold region
of the ReLu. We prove the following theorem, in which the parameter $\tau$ in our
theorem controls the trade-off between the polynomial $p_d$ being close to 0 for $x < 0$
and being close to $x$ for $x > 0$.
\begin{thm}\label{thm:relu-approximant}
  Suppose $R > 0$, $0 < \tau < 1/2$ and $d \ge 7$.
  Then there exists a polynomial $p_d = p_{d,\tau,R}$ of degree $d$ such that for $x \in [-R,0]$
  \[ |p_d(x) - \mbox{ReLU}(x)| \le 14 R\sqrt{\frac{d}{\tau\pi}} e^{-\sqrt{\pi \tau d/4}} \]
  and for $x \in [0,R]$,
\[ |p_d(x) - \mbox{ReLU}(x)| \le 2R\tau + 2R\sqrt{\frac{4 \tau}{\pi d}} + 12R\sqrt{\frac{d}{\tau\pi}}e^{-\sqrt{\pi \tau d/4}}. \]  
\end{thm}
The proof proceeds by combining a mollification of ReLU with complex analytic machinery from approximation theory. Before presenting it, 
let us see how it can be used to imply the main result, Theorem \ref{thm:poly-upper-bound-intro}. 

Toward that, we substitute our polynomial construction for $\rho_{\tau}$ in Lemma~\ref{lem:thresholding-regression}. Namely, define $M_{\tau} = M + 2\tau$ and
\[ \tilde{\rho}_{d,\tau,M} = p_{d,\tau/M_{\tau},M_{\tau}}(x - \tau) + p_{d,\tau/M_{\tau},M_{\tau}}(-x + \tau) \]
where $p$ is the polynomial constructed in Theorem~\ref{thm:relu-approximant}. We then have: 
\begin{lem}\label{lem:polynomial-soft-threshold}
  Suppose $\epsilon,\tau > 0$ and $M \ge 1$. Then for all $d \ge d_0 = \Omega(\frac{M_{\tau}}{\tau}\log^2(\frac{M_{\tau}}{\epsilon \tau}))$, for $|x| \in (\tau, M_{\tau})$ we have
  \[ |\tilde{\rho}_{d,\tau,M}(x) - x| \le 3\tau + \epsilon \]
  and for $|x| \le \tau$ we have
  \[ |\tilde{\rho}_{d,\tau,M}(x)| \le \epsilon \]
\end{lem}
\begin{proof}
  By the guarantee of Theorem~\ref{thm:relu-approximant}, we see that for for $|x| \le \tau$ that
  \[ |\tilde{\rho}_{d,\tau,M}(x)| \le 28 M_{\tau} \sqrt{\frac{d M_{\tau}}{\tau\pi}}e^{-\sqrt{\pi \tau d/4M_{\tau}}}. \]  
  Thus we see that taking $d = \Omega(\frac{M_{\tau}}{\tau}\log^2(\frac{M_{\tau}}{\epsilon \tau}))$ suffices to make the latter expression at most $\epsilon$.
  Similarly for $|x| > \tau$ we know that
  \[ |\tilde{\rho}_{d,\tau,M}(x)| \le 2\tau + 2M_{\tau}\sqrt{\frac{4 \tau}{M_{\tau} \pi d}} + 26 M_{\tau} \sqrt{\frac{d M_{\tau}}{\tau\pi}}e^{-\sqrt{\pi \tau d/4M_{\tau}}} \]
  and taking $d = \Omega(\frac{M_{\tau}}{\tau}\log^2(\frac{M_{\tau}}{\epsilon \tau}))$ with sufficiently large constant guarantees the middle term is at most $\tau$ and the last term is at most $\epsilon$.
\end{proof}

Using this, we can show that if we use a polynomial of degree $\Omega((M/\sigma \sqrt{\log n}) \log^2 m )$ we can achieve similar statistical performance to
  the ReLu network. Namely, we can show:
\begin{lem}
  Suppose $A$ is $\mu$-incoherent i.e. $\|A^{\top} A - Id\|_{\infty} \le \mu$. Let $z$   be an arbitrary fixed vector such that $\|z\|_1 \le M$ and $|\supp(z)| \le k$.
  Suppose $x = A z + \xi$ where $\xi \sim N(0, \sigma^2 Id_{n \times n})$.
  Then for some $\tau = \Theta(\sigma \sqrt{(1 + \mu) \log m} + \mu M)$, for any $d \ge d_0 = \Omega(\frac{M_{\tau}}{\tau} \log^2(M_{\tau} m/\tau^2))$,
  if we take $\hat{z} := \tilde{\rho}_{d,\tau,M}^{\otimes n}(A^{\top} x)$, then with high probability we have $\|\hat{z} - z\|_1 \le 6k\tau$.
\end{lem}
\begin{proof}
  Apply Lemma~\ref{lem:polynomial-soft-threshold} with $\epsilon = \tau/m$. Then we see
  for $|x| \in (\tau, M_{\tau})$ we havn
  \[ |\tilde{\rho}_{d,\tau,M}(x) - x| \le (3 + 1/m)\tau \le 4 \tau \]
  and for $|x| \le \tau$ we have
  \[ |\tilde{\rho}_{d,\tau,M}(x)| \le \tau/m \]    
  Observe that
  \[ A^{\top} x = z + (A^{\top} A - Id)z + A^{\top} \xi. \]  
  Note that entry $i$ of $A^{\top} \xi$ is $\langle A_i, \xi \rangle$ where
  $\|A_i\|_2^2 \le (1 + \mu)$ so $(A^t \xi)_i$ is Gaussian with variance
  at most $\sigma^2(1 + \mu)$.

  By choosing $\tau$ with sufficiently large constant, then applying the sub-Gaussian tail bound and union bound, with high probability all
  coordinates not in the true support are thresholded to at most $\tau/m$. Similarly we see
  that for each of the coordinates in the support, an error of at most $5\tau$
  is made. Therefore $\|\hat{z} - z\|_1 \le 5k \tau + m(\tau/m) \le 6k \tau$.
\end{proof}

Now we have all the ingredients to prove Theorem~\ref{thm:poly-upper-bound-intro}: 
\begin{proof}[Proof of Theorem \ref{thm:poly-upper-bound-intro}] 
Define an estimate for $Y$ by taking $\hat{Z}_{d,M} := \tilde{\rho}_{d,\tau,M}^{\otimes n}(A^{\top} X)$ where $\tau$ is defined as in the Lemma, and then taking
$ \hat{Y}_{d,M} := \langle w, \hat{Z}_{d,M} \rangle $. Applying the previous Lemma, we get analogous versions of Theorem~\ref{thm:2layer} by the same argument as in that theorem.
\end{proof}

Finally, we return to the proof of Theorem~\ref{thm:relu-approximant}: 
\begin{proof}[Proof of Theorem~\ref{thm:relu-approximant}]
  We start with the case where $R = 1/2$.
  We build the approximation in two steps. First we approximate ReLu by
  the following ``annealed'' version of ReLu, for parameters $\beta > \pi,\tau > 0$
  to be optimized later:
  \[ g_{\beta}(x) = \frac{1}{\beta} log(1 + e^{\beta x}) \]
  \[ f_{\beta,\tau}(x) = g_{\beta}(x - \tau). \]
  Observe that when we look at negative inputs, $g_{\beta}(-x) = \frac{1}{\beta} \log(1 + e^{-\beta x}) \le \frac{1}{\beta} e^{-\beta x}$.
  Therefore when $x < 0$, $f_{\beta}(x) \le \frac{1}{\beta} e^{-\beta \tau}$.

  For the second step,, we need to show $f_{\beta}$ can be well-approximated by low-degree polynomials. In fact, because $f_{\beta}$
  is analytic in a neighborhood of the origin, it turns out that its optimal rate of approximation is determined
  exactly by its complex-analytic properties. More precisely, define $D_{\rho}$ to be
  the region bounded by the ellipse in $\mathbb{C} = \mathbb{R}^2$ centered at the origin with equation
  \[ \frac{x^2}{a^2} + \frac{y^2}{b^2} = 1 \] with semi-axes
  $a = \frac{1}{2}(\rho + \rho^{-1})$ and
  $b = \frac{1}{2} |\rho - \rho^{-1}|$; the focii of the ellipse are
  $\pm 1$.  For an arbitrary function $f : [-1,1] \to \mathbb{R}$, let $E_d(f)$ denote the error of the best polynomial
  approximation of degree $d$ in infinity norm on the interval $[-1,1]$ of $f$. Then the following
  theorem of Bernstein exactly characterizes the growth rate of $E_d(f)$:
  \begin{thm}[Theorem 7.8.1, \cite{constructive-approximation}]
    Let $f$ be a function defined on $[-1,1]$. Let $\rho_0$ be the supremum of all $\rho$ such that
    $f$ has an analytic extension on $D_{\rho}$. Then
\[ \limsup_{d \to \infty} \sqrt[d]{E_d(f)} = \frac{1}{\rho_0} \]    
  \end{thm}
  For our application we need only the upper bound and we need a quantitative estimate for finite $n$.
  Following the proof of the upper bound in \cite{constructive-approximation}, we get the following result:
\begin{thm}\label{bernstein-quantitative}
  Suppose $f$ is analytic on the interior of $D_{\rho_1}$ and $|f(z)| \le M$ on the closure of $D_{\rho_1}$.  Then
\[ E_d(f) \le \frac{2M}{\rho_1 - 1} \rho_1^{-n} \]
\end{thm}
We will now apply this theorem to $g_{\beta}$. First, we claim that $g_{\beta}$ is analytic on $D_{\rho_1}$ where
$\rho_1$ is the solution to this equation for the semi-axis of the ellipse:
\[ \frac{1}{2} (\rho - \rho^{-1}) = \frac{\pi}{2\beta} \]
which is
\[ \rho_1 = \frac{\sqrt{4 \beta^2 + \pi^2} + \pi}{2 \beta} > 1 + \pi/2\beta. \]
To see this, first extend $\log$ to the complex plane by taking a branch cut at $(-\infty,0]$. To prove
$g_{\beta}$ is analytic on $D_{\rho_1}$, we just need to prove that $1 + e^{\beta z}$ avoids $(-\infty,0]$
for $z \in D_{\rho_1}$. This follows because by the definition of $\rho_1$, for every $z \in D_{\rho_1}$,
$\Im(z) < \frac{\pi}{2 \beta}$ hence $\Re(1 + e^{\beta z}) \ge 1$. We also see that for $z \in D_{\rho_1}$,
\[ |g_{\beta}(z)| = \frac{1}{\beta} |\log(1 + e^{\beta z})| \le \frac{1}{\beta} \sup_{w \in D_{\beta \rho_1}} |\log(1 + e^{w})| \le \frac{1}{\beta}(\log(1 + e^{\beta}) + \pi) < 6. \]
Therefore by Theorem~\ref{bernstein-quantitative} we have
\[ E_d(g_{\beta}) \le \frac{12 \beta}{\pi} (1 + \pi/2\beta)^{-n} \le \frac{12 \beta}{\pi} e^{-\pi n/4 \beta} \]
where in the last step we used that $1 + x \ge \exp(x/2)$ for $x < 1/2$ and that $\beta > \pi$.
Let $\tilde{g}_{\beta,d}$ denote the best polynomial approximation to $g_{\beta}$ of degree $d$ and let
$\tilde{f}_{\beta,\tau,d} = \tilde{g}_{\beta,d}(x - \tau)$

Thus for $x \in [-1 + \tau,0]$,
\[ |ReLu(x) - \tilde{f}_{\beta,\tau,d}(x)| \le |f_{\beta,\tau}(x)| + |\tilde{g}_{\beta,d}(x - \tau) - g_{\beta,\tau}(x - \tau)| \le \frac{1}{\beta} e^{-\beta \tau} + \frac{12 \beta}{\pi} e^{-\pi d/4 \beta} \]
Take $\beta = \sqrt{\pi d/4\tau}$ and require $d > 7$ so that $\beta > 1$, then for $x \in [-1 + \tau,0]$,
\[ |ReLu(x) - \tilde{f}_{\beta,\tau,d}(x)| \le 7\sqrt{\frac{d}{\tau\pi}} e^{-\sqrt{\pi \tau d/4}} \]
For $x \in (0, 1 - \tau]$ we have by the 1-Lipschitz property of $g_{\beta}$ and calculus that
\[ |x - f_{\beta,\tau}(x)| \le \tau + |x - g_{\beta}(x)| \le \tau + \frac{\log 2}{\beta} \]
so
\[ |ReLu(x) - \tilde{f}_{\beta,\tau,d}(x)| \le |x - f_{\beta,\tau}(x)| + |\tilde{g}_{\beta,d}(x - \tau) - g_{\beta,\tau}(x - \tau)| \le \tau + \frac{\log 2}{\beta} + \frac{12 \beta}{\pi} e^{-\pi d/4 \beta}.\]
Plugging in our value of $\beta$ and using $\log 2 \le 1$ gives
\[ |ReLu(x) - \tilde{f}_{\beta,\tau,d}(x)| \le \tau + \sqrt{\frac{4 \tau}{\pi d}} + 6\sqrt{\frac{d}{\tau\pi}}e^{-\sqrt{\pi \tau d/4}} \]
Now the result for general $R$ follows by taking $p_d(x) = 2R \tilde{f}_{\beta,\tau,d}(x/2R)$, since $2R \cdot ReLu(x/2R) = ReLu(x)$ and $[-1/2,1/2] \subset [-1 + \tau, 1 - \tau]$. 
\end{proof}

\section{Conclusions}

We've attacked the problem of providing representation lower and upper bounds 
for different classes of universal approximators in a natural statistical setup. 
We hope this will inspire researches to move beyond the worst-case setup when 
considering the representational power of different predictors. 

The techniques we develop are interesting in their own right: unlike standard approximation 
theory setups, we need to design polynomials which may only need to be accurate in 
certain regions. Conceivably, in classification setups, similar wisdom may be helpful: 
the approximator needs to only be accurate near the decision boundary. 

Finally, we conclude with a tantalizing open problem: 
%We have obtained almost tight upper and lower bounds for the
%representational power of polynomial kernels in the setting in which
%$Z$ is $k$-sparse for small $k$ with $k << n$, and shown a similar
%upper bound using a simple 2-layer ReLU network. 
In general it is possible to obtain non-trivial sparse recovery guarantees for
LASSO even when the sparsity $k$ is nearly of the same order as $n$
under assumptions such as RIP. Since LASSO can be computed quickly using iterated soft thresholding (ISTA
and FISTA, see \cite{beck2009fast}), we see that sufficiently deep
neural networks can compute a near-optimal solution in this setting
as well. It would be interesting to determine whether shallower networks
and polynomials of degree $\mbox{polylog}(n)$ can achieve a similar guarantee.

\bibliographystyle{plainnat}
\bibliography{bib}

\begin{thebibliography}{19}
\providecommand{\natexlab}[1]{#1}
\providecommand{\url}[1]{\texttt{#1}}
\expandafter\ifx\csname urlstyle\endcsname\relax
  \providecommand{\doi}[1]{doi: #1}\else
  \providecommand{\doi}{doi: \begingroup \urlstyle{rm}\Url}\fi

\bibitem[Bartlett et~al.(2017)Bartlett, Foster, and
  Telgarsky]{bartlett2017spectrally}
Peter~L Bartlett, Dylan~J Foster, and Matus~J Telgarsky.
\newblock Spectrally-normalized margin bounds for neural networks.
\newblock In \emph{Advances in Neural Information Processing Systems}, pages
  6241--6250, 2017.

\bibitem[Beck and Teboulle(2009)]{beck2009fast}
Amir Beck and Marc Teboulle.
\newblock A fast iterative shrinkage-thresholding algorithm for linear inverse
  problems.
\newblock \emph{SIAM journal on imaging sciences}, 2\penalty0 (1):\penalty0
  183--202, 2009.

\bibitem[Chang et~al.(2010)Chang, Hsieh, Chang, Ringgaard, and
  Lin]{chang2010training}
Yin-Wen Chang, Cho-Jui Hsieh, Kai-Wei Chang, Michael Ringgaard, and Chih-Jen
  Lin.
\newblock Training and testing low-degree polynomial data mappings via linear
  svm.
\newblock \emph{Journal of Machine Learning Research}, 11\penalty0
  (Apr):\penalty0 1471--1490, 2010.

\bibitem[DeVore and Lorentz(1993)]{constructive-approximation}
Ronald~A DeVore and George~G Lorentz.
\newblock \emph{Constructive approximation}, volume 303.
\newblock Springer Science \& Business Media, 1993.

\bibitem[Diakonikolas et~al.(2010)Diakonikolas, Kane, and
  Nelson]{ft-mollification}
Ilias Diakonikolas, Daniel~M Kane, and Jelani Nelson.
\newblock Bounded independence fools degree-2 threshold functions.
\newblock In \emph{Foundations of Computer Science (FOCS), 2010 51st Annual
  IEEE Symposium on}, pages 11--20. IEEE, 2010.

\bibitem[Eldan and Shamir(2016)]{eldan2016power}
Ronen Eldan and Ohad Shamir.
\newblock The power of depth for feedforward neural networks.
\newblock In \emph{Conference on Learning Theory}, pages 907--940, 2016.

\bibitem[Hardt et~al.(2016)Hardt, Recht, and Singer]{hardt2015train}
Moritz Hardt, Ben Recht, and Yoram Singer.
\newblock Train faster, generalize better: Stability of stochastic gradient
  descent.
\newblock In \emph{International Conference on Machine Learning}, pages
  1225--1234, 2016.

\bibitem[Livni et~al.(2014)Livni, Shalev-Shwartz, and Shamir]{livni-et-al}
Roi Livni, Shai Shalev-Shwartz, and Ohad Shamir.
\newblock On the computational efficiency of training neural networks.
\newblock In \emph{Advances in Neural Information Processing Systems}, pages
  855--863, 2014.

\bibitem[Moitra(2018)]{moitra2014algorithmic}
Ankur Moitra.
\newblock \emph{Algorithmic aspects of machine learning}.
\newblock Preprint. Cambridge University Press (to appear), 2018.

\bibitem[Newman et~al.(1964)]{newman1964rational}
Donald~J Newman et~al.
\newblock Rational approximation to $| x| $.
\newblock \emph{The Michigan Mathematical Journal}, 11\penalty0 (1):\penalty0
  11--14, 1964.

\bibitem[Poggio et~al.(2017)Poggio, Mhaskar, Rosasco, Miranda, and
  Liao]{poggio2017}
Tomaso Poggio, Hrushikesh Mhaskar, Lorenzo Rosasco, Brando Miranda, and Qianli
  Liao.
\newblock Why and when can deep-but not shallow-networks avoid the curse of
  dimensionality: A review.
\newblock \emph{International Journal of Automation and Computing}, 14\penalty0
  (5):\penalty0 503--519, 2017.

\bibitem[Rigollet(2017)]{rigollet-hdstats}
Phillippe Rigollet.
\newblock High-dimensional statistics.
\newblock \emph{Lecture notes (MIT)}, 2017.

\bibitem[Safran and Shamir(2017)]{safran2017depth}
Itay Safran and Ohad Shamir.
\newblock Depth-width tradeoffs in approximating natural functions with neural
  networks.
\newblock In \emph{International Conference on Machine Learning}, pages
  2979--2987, 2017.

\bibitem[Shalev-Shwartz et~al.(2011)Shalev-Shwartz, Shamir, and
  Sridharan]{shalev2011learning}
Shai Shalev-Shwartz, Ohad Shamir, and Karthik Sridharan.
\newblock Learning kernel-based halfspaces with the 0-1 loss.
\newblock \emph{SIAM Journal on Computing}, 40\penalty0 (6):\penalty0
  1623--1646, 2011.

\bibitem[Telgarsky(2016)]{telgarsky2016benefits}
Matus Telgarsky.
\newblock Benefits of depth in neural networks.
\newblock In \emph{Conference on Learning Theory}, pages 1517--1539, 2016.

\bibitem[Telgarsky(2017)]{telgarsky2017neural}
Matus Telgarsky.
\newblock Neural networks and rational functions.
\newblock In \emph{International Conference on Machine Learning}, pages
  3387--3393, 2017.

\bibitem[Vershynin(2018)]{vershynin2016high}
Roman Vershynin.
\newblock \emph{High-Dimensional Probability}.
\newblock Cambridge University Press (to appear), 2018.

\bibitem[Yarotsky(2017)]{yarotsky2017error}
Dmitry Yarotsky.
\newblock Error bounds for approximations with deep relu networks.
\newblock \emph{Neural Networks}, 94:\penalty0 103--114, 2017.

\bibitem[Zhang et~al.(2017)Zhang, Bengio, Hardt, Recht, and
  Vinyals]{zhang2016understanding}
Chiyuan Zhang, Samy Bengio, Moritz Hardt, Benjamin Recht, and Oriol Vinyals.
\newblock Understanding deep learning requires rethinking generalization.
\newblock \emph{ICLR}, 2017.

\end{thebibliography}

\newpage
\appendix
%{\Large \textbf{Supplementary Material for ``Representational power of ReLU networks and polynomial kernels: beyond worst-case analysis"}}
%\setcounter{page}{1}

%%% Local Variables:
%%% mode: latex
%%% TeX-master: "main"
%%% End:
\section{Upper bound via 2-Layer ReLu Network}\label{sec:relu-upperbound}
First we describe the 2-layer ReLu network that we will analyze.
Define the soft threshold function with threshold $\tau$: $\rho_{\tau}(x) = \sgn(x)\tau \min(0, |x| - \tau) = \mbox{ReLu}(x - \tau) + \mbox{ReLu}(-x + \tau)$.
Then we consider the following estimate for $y$, which corresponds to a 2-layer neural network:
\[ \hat{Z}_{NN} := \rho_{\tau}^{\otimes n}(A^{\top} X) \]
\[ \hat{Y}_{NN} := \langle w, \hat{Z}_{NN} \rangle \] We will first
prove a bound on the error of the soft-thresholding estimator
$\hat{Z}_{NN}$ (Lemma~\ref{lem:thresholding-regression}), which
corresponds to the hidden layer of the neural network and is
essentially a standard fact in high-dimensional statistics (see
reference text \cite{rigollet-hdstats}). The idea is that the
soft thresholding will correctly zero-out most of the coordinates
in the support while adding only a small additional error
to the coordinates outside the support.

From the recovery guarantee for $\hat{Z}_{NN}$, we will then deduce the
following theorem for the estimator $\hat{Y}_{NN}$:
\begin{thm}\label{thm:2layer}
With high probability, the estimator  $\hat{Y}_{NN}$ satisfies 
\[ |\hat{Y}_{NN} - Y|^2 = O((1 + \mu) \sigma^2 k^2 \log(m) + \mu^2 k^2 M^2) \]
\end{thm}
In order to interpret this, note that standard constructions of
incoherent matrices, such as a matrix with independent random
$\pm 1/\sqrt{n}$ entries, give incoherence of
$\mu = O(\sqrt{\frac{\log m}{n}})$ (this follows from concetration
inequalities and a union bound. See e.g. reference text
\cite{rigollet-hdstats}). Therefore for such an incoherent matrix, and
a choice of $k$ which is not too large with respect to $n$, the effect
of $\mu$ in the bound is small and can be disregarded. Then this bound
is intuitive because if we think of $k$ as small and fixed, it
says the error is on the order of the noise $\sigma^2$ with an
additional log factor for not knowing where the true support lies.

Towards proving the above result, we first need an estimate on the bias of $A^{\top} x$, i.e. the error without noise:
\begin{lem}
  Suppose $A$ is $\mu$-incoherent i.e. $\|A^{\top} A - Id\|_{\infty} \le \mu$. Then for any $z$, $\|A^{\top} A z - z\|_{\infty} \le \mu \|z\|_1$.
\end{lem}
\begin{proof}
  \[ (A^{\top} A z)_i = \langle A_i, \sum_j z_j A_j \rangle = z_i \langle A_i, A_i \rangle + \sum_{j \ne i} z_j \langle A_i,A_j \rangle \]
  so applying the incoherence assumption we have $|(A^{\top} A z)_i - z_i| \le \mu \|z\|_1$.
\end{proof}
Using this we can analyze the error in thresholding.
\begin{lem}\label{lem:thresholding-regression}
  Suppose $A$ is $\mu$-incoherent i.e. $\|A^{\top} A - I\|_{\infty} \le \mu$. Let $z$   be an arbitrary fixed vector such that $\|z\|_1 \le M$ and $|\supp(z)| \le k$.
  Suppose $x = A z + \xi$ where $\xi \sim N(0, \sigma^2 I_{n \times n})$.
  Then for some $\tau = \Theta(\sigma \sqrt{(1 + \mu) \log m} + \mu M)$ and $\hat{z} = \rho_{\tau}^{\otimes n}(A^{\top} x)$, with high probability we have $\|\hat{z} - z\|_{\infty} \le 2\tau$ and $\supp(\hat{z}) \subset \supp(z)$.
\end{lem}
\begin{proof}
  Observe that
  \[ A^{\top} x = z + (A^{\top} A - I)z + A^{\top} \xi. \]  
  Note that entry $i$ of $A^{\top} \xi$ is $\langle A_i, \xi \rangle$ where
  $\|A_i\|_2^2 \le (1 + \mu)$ so $(A^t \xi)_i$ is subgaussian with variance proxy
  at most $\sigma^2(1 + \mu)$.

  By concentration and union bound, with high probability all
  coordinates not in the true support are thresholded to 0. Similarly we see
  that for each of the coordinates in the support, an error of at most $2\tau$
  is made.
\end{proof}
From the above lemma, we can easily prove the main theorem of this section: 
\begin{proof} [Proof of Theorem~\ref{thm:2layer}]
  When the high probability above event happens, we have the following upper bound by Holder's inequality:
  \[ |\hat{Y}_{NN} - Y|^2 = \langle w|_{\supp(h)}, (\hat{Z}_{NN} - Z)|_{\supp(h)} \rangle^2 \le k^2 \|\hat{Z}_{NN} - Z\|_{\infty}^2 = O(k^2 ((1 + \mu) \sigma^2 \log(m) + \mu^2 M^2)) \]
%  When the high probability above event happens, we have
%\[ |\hat{Y}_{NN} - Y|^2 = \langle w|_{\supp(h)}, (\rho_{\tau}^{\otimes n}(A^{\top} X) - Z)|_{\supp(h)} \rangle^2 \]
\end{proof}

% Q: can we fix the discrepancy without the assumptions?
For the lower bounds we will be interested mostly in the case when
$\mu = 0$, i.e. $A$ is orthogonal and so $m = n$, the coordinates
of $Z$ are independent and each is nonzero with probability at most $k/n$, and the noise is Gaussian.
Then the error estimate we had in the previous theorem specializes
to $O(\sigma^2 k^2 \log(n))$, but under these assumptions we know that the
information-theoretic optimal is actually $\sigma^2 k \log(n)$. We can
redo the analysis to eliminate the extra factor of $k$, without
changing the algorithm:
\begin{thm}\label{thm:predicting-Y}
  Suppose $A$ is orthogonal (hence $m = n$), the coordinates of $Z$
  are independent, and $\xi \sim N(0,\sigma^2 I)$. Then
\[ \E |\hat{Y}_{NN} - Y|^2 = O(k\sigma^2 \log(m)) \]  
\end{thm}
\begin{proof}
  In this case, we have $A^{\top} X = Z + \xi'$ where
  $\xi'\sim N(0,\sigma^2 I)$. Therefore the coordinates of $\hat{Z}$
  are independent of each other, and so we see
  \[ \E|\hat{Y}_{NN} - Y|^2 = \sum_i w_i^2 \E[(\hat{Z}_{NN} - Z)_i^2] \le \sum_i \E[(\hat{Z}_{NN} - Z)_i^2]. \]
  Let $\mathcal{E}_i$ denote the event that $|\xi'|_i > \tau$. Then
  \begin{align*}
    \sum_i \E[(\hat{Z}_{NN} - Z)_i^2]
    &= \sum_i \E[(\1_{\mathcal{E}_i} + \1_{\mathcal{E}_i^C}) (\hat{Z}_{NN} - Z)_i^2] \\
    &\le 4k \tau^2 + \sum_i \E[\1_{\mathcal{E}_i^C} (\hat{Z}_{NN} - Z)_i^2] \\
    &= 4k \tau^2 + \sum_i \Pr(\1_{\mathcal{E}_i^C}) \E[(\hat{Z}_{NN} - Z)_i^2 | \1_{\mathcal{E}_i^C} = 1] \\
    &\le 4k \tau^2 + \sum_i \Pr(\1_{\mathcal{E}_i^C}) \E[(\tau + |X'_i - Z_i|)^2 | \1_{\mathcal{E}_i^C} = 1] \\
    &\le 4k \tau^2 + \sum_i \Pr(\1_{\mathcal{E}_i^C}) (2\tau^2 + 2\E[|\xi'_i|^2 | \1_{\mathcal{E}_i^C} = 1]) \\
    &\le 4k \tau^2 + \sum_i \frac{C}{m} (2 \tau^2 + 2 C' \tau^2)
  \end{align*}
  where the first inequality follows as in
  Lemma~\ref{lem:thresholding-regression}, the second inequality uses
  that $|\rho_{\tau}(x) - x| \le \tau$, the third uses that
  $(a + b)^2 = a^2 + 2ab + b^2 \le 2a^2 + 2b^2$ by Young's inequality, and the last
  inequality follows from standard tail bounds on Gaussians. We see
  the last expression is $O(k \sigma^2 \log(m))$ so we have proved
  the result.
  % FIXME: add proof details
\end{proof}
%\Fnote{why do we have a discrepancy in $k^2$ vs $k$ in the upper and lower bound. This must be due to something trivial/inconsistent assumptions.}

%%% Local Variables:
%%% mode: latex
%%% TeX-master: "main"
%%% End:

%\input{bias-variance} 
%\input{polynomial-upper-bound-appendix}

\end{document}